\documentclass[final]{cvpr}

\usepackage{times}
\usepackage{epsfig}
\usepackage{graphicx}
\usepackage{amsmath}
\usepackage{amssymb}

\usepackage[pagebackref=true,breaklinks=true,colorlinks,bookmarks=false]{hyperref}
\usepackage{algorithmic}
\usepackage[ruled,vlined]{algorithm2e}
\usepackage{cleveref}
\usepackage{comment}
\usepackage{capt-of}
\usepackage{booktabs}
\usepackage{multirow}
\usepackage{verbatim}

\usepackage{grffile}

\usepackage{amsthm}
\theoremstyle{definition}

\newtheorem{theorem}{Theorem}

\newcommand\sizeImage{0.13}

\newcommand\sizeHorizontalSpace{-7pt}

\def\cvprPaperID{****} % *** Enter the CVPR Paper ID here
\def\confYear{CVPR 2021}

\usepackage{amsmath}
\usepackage{amsfonts,dsfont,bm,mathrsfs}
\usepackage{graphicx}
\usepackage{hyperref}
\usepackage{booktabs}
\usepackage{multirow}
\usepackage{subcaption}
\usepackage{grffile}

\newcommand{\R}{\mathbb{R}}      % notation of real number space.
\newcommand{\z}{{\rm\bf z}}      % notation of latent code.
\newcommand{\Z}{\mathcal{Z}}     % notation of latent space.
\newcommand{\x}{{\rm\bf x}}      % notation of image.
\newcommand{\X}{\mathcal{X}}     % notation of image space.
\begin{document}

\title{Differentially Private Imaging via Latent Space Manipulation}

\author{

    Tao Li and Chris Clifton \\
    {\normalsize Department of Computer Science, Purdue University, West Lafayette, Indiana, USA} \\
    {\tt\small \{taoli,clifton\}@purdue.edu}
}

\maketitle

\begin{abstract}
{
There is growing concern about image privacy
due to the popularity of social media and photo devices, along with increasing use of face recognition systems. 
However, established image de-identification techniques are either too subject to re-identification, produce photos that are insufficiently realistic, or both.
To tackle this, we present a novel approach for image obfuscation by manipulating latent spaces of an unconditionally trained generative model that is able to synthesize photo-realistic facial images of high resolution.
This manipulation is done in a way that satisfies the formal privacy standard of local differential privacy.
To our knowledge, this is the first approach to image privacy that satisfies $\varepsilon$-differential privacy \emph{for the person.}
}
\end{abstract}

\section{Introduction}\label{sec:introduction}
Image obfuscation techniques have been used to protect sensitive information, such as human faces and confidential texts.
However, recent advances in machine learning, especially deep learning, make standard obfuscation methods such as pixelization and blurring less effective at protecting privacy \cite{mcpherson2016defeating};
it has been showed that over 90\% of blurred faces can be re-identified by deep convolutional neural networks or commerical face recognition systems \cite{li2021deepblur}.

We envision scenarios where the image should convey the general tone and activity (e.g., facial expressions), but not identify individuals.  For example, one could post photos on social media retaining images of friends, but protecting identity of bystanders while maintaining the general feel of the image; an example is given in \cref{fig:example}.  In a very different scenario, surveillance footage could be viewed by police to identify suspicious acts, but identity of those in the image would only be available through appropriate court order, protecting against (possibly unintended) profiling and ``guilt by association''.  In both scenarios, blurring/pixelization fails to preserve desired semantics (e.g., facial expression), and fails to provide the desired privacy protection.

Many attempts have been made to obfuscate images and some privacy guarantees are provided.
A pixelization method proposed in \cite{fan2018image} satisfies pixel-wise $\epsilon$-differential privacy~\cite{dwork2006calibrating}.
However, the utility of the pixelized images is far from satisfactory, due to the high perturbation noise needed to reasonably hide the original; the images appear like traditional pixelization or blurring techniques.
A more serious problem is that this provides differential privacy for \emph{pixels}, not for the individuals pictured in the image.  Not only are the images highly distorted, as with ad-hoc approaches to pixelization and blurring they are subject to re-identification of the individuals in the image \cite{gross2005integrating,gross2006model,gross2009face}.

In this paper, we show how differential privacy can be provided at the level of the individual in the image.  The key idea is that we transform the image into a semantic latent space.  We then add random noise to the latent space representation in a way that satisfies $\varepsilon$-differential privacy.  We then generate a new image from the privatized latent space representation.  This ensures a formal privacy guarantee, while providing an image that preserves important characteristics of the original.

\begin{figure}[t]
\begin{center}
    \includegraphics[width=\sizeImage\textwidth]{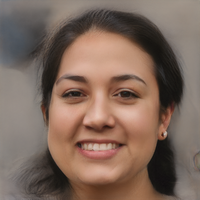}
    \includegraphics[width=\sizeImage\textwidth]{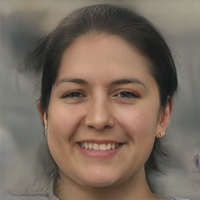}
    \includegraphics[width=\sizeImage\textwidth]{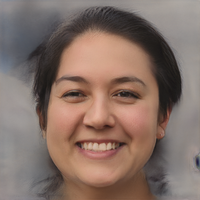}
    \includegraphics[width=\sizeImage\textwidth]{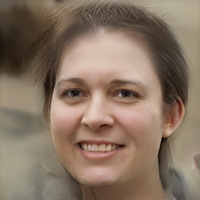}
    \includegraphics[width=\sizeImage\textwidth]{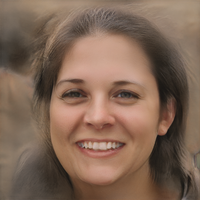}
    \includegraphics[width=\sizeImage\textwidth]{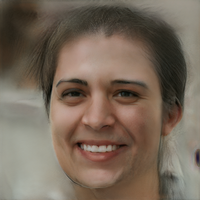}
\end{center}
\caption{Can you identify the authors?
    These are images of the authors, with noise added %using the presented mechanism
    that satisfies %the rigorous definition of
    differential privacy sufficient to prevent identification of the authors if you do not already know who they are.
}
\label{fig:teaser}
\end{figure}

\begin{figure*}[th]
\captionsetup[subfigure]{justification=centering}
\centering
    \begin{subfigure}{0.40\textwidth}
        \includegraphics[width=\textwidth]{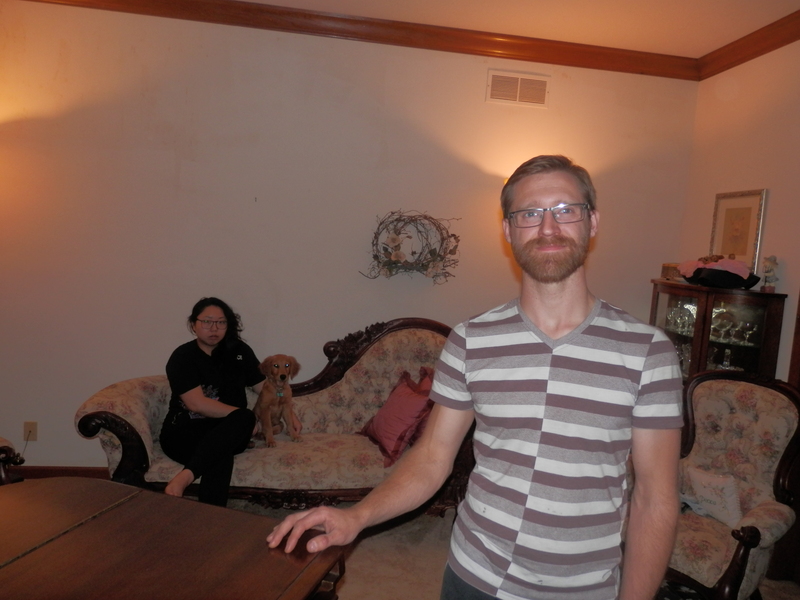}
        \caption{Original}
        \label{fig:example-original}
    \end{subfigure}
    ~
    \begin{subfigure}{0.40\textwidth}
        \includegraphics[width=\textwidth]{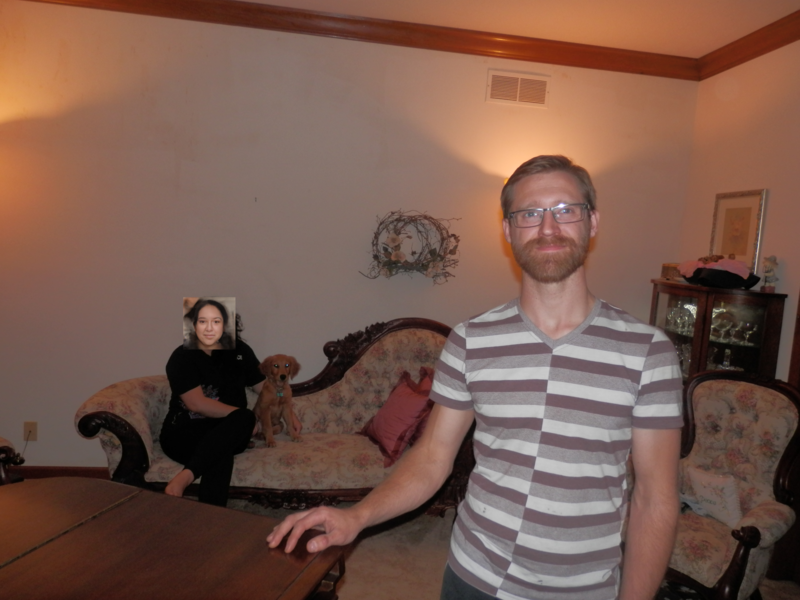}
        \caption{Differentially private bystander}
        \label{fig:example-anonymized}
    \end{subfigure}
\caption{Protecting bystanders on social media.  The person in the background has been replaced with a differentially private version, while the subject of the image is maintained.  Note that in real use, the background and lighting would be blended (as a postprocessing step); for clarity we are showing only the facial image manipulation.}
\label{fig:example}
\end{figure*}

Other obfuscation methods have been proposed recently to balance privacy and utility. 
For example, adding noise to an SVD-transformation is proposed in~\cite{fan2019practical}; however, the approach does not formalize privacy in the sense of identifying \emph{individuals}.
\cite{ren2018learning} makes use of generative adversarial networks to obfuscate a face in the context of detecting and depicting (anonymized) actions.
\cite{sun2018natural} adopted generative adversarial networks (GANs) for facial image obfuscation by identifying a face and accordingly inpainting it with a synthesized face alternative.
This unfortunately has the potential to lose important characteristics of the original image.
Another approach leveraged a conditional GAN to manipulate facial attributions in accordance with different privacy requirements~\cite{li2019anonymousnet}.
These approaches suffer the common failing that they do not provide a \emph{formal} privacy guarantee.  As such, they may be subject to re-identification or re-construction attacks.

Building on top of previous works, this paper presents a practical image obfuscation method with provable guarantees and some level of photo-realism.
Unlike \cite{sun2018natural} which replaces the entire face with an arbitrary substitute, and \cite{li2019anonymousnet} which obfuscates facial images on a discrete attribute space, this work further extends facial image manipulation to a continuous latent space.  Applying  differential privacy in this latent semantic space provides greater photo-realism while satisfying rigorous privacy guarantees.

A key to formal privacy methods is that there is randomness in the approach:
the same image privatized twice will not look the same (as demonstrated in \cref{fig:teaser},
which includes multiple images of some authors.)
This randomness is a key component to preventing reconstruction attacks. % This sentence may belong later in the paper.
We show that randomized manipulations in the latent semantic space can be expected to provide realistic images.
The method guarantees similarity-based indistinguishability among images, providing privacy
guarantees in worst-case scenarios and boosting the utility
of the obfuscated image.

Our main contributions are:
\begin{itemize}
    \item The first definition of $\varepsilon$-differential privacy for images that protects \emph{individuals} in the image;
    \item A practical framework for real-world differentially private imaging that maintains a level of image semantics;
    \item
        We introduce a clipping step in image latent space that enables a formal guarantee of $\varepsilon$-differential privacy with significantly improved fidelity.
\end{itemize}

In the rest of the paper:
\Cref{sec:literature} discusses related works;
We formalize our approach in \cref{sec:method} and show how it satisfies differential privacy with a practical framework;
\Cref{sec:experiment} details our implementation and demonstrates the results;
\Cref{sec:conclusion} concludes the paper.

\begin{figure*}[t]
\centering
\includegraphics[width=\textwidth]{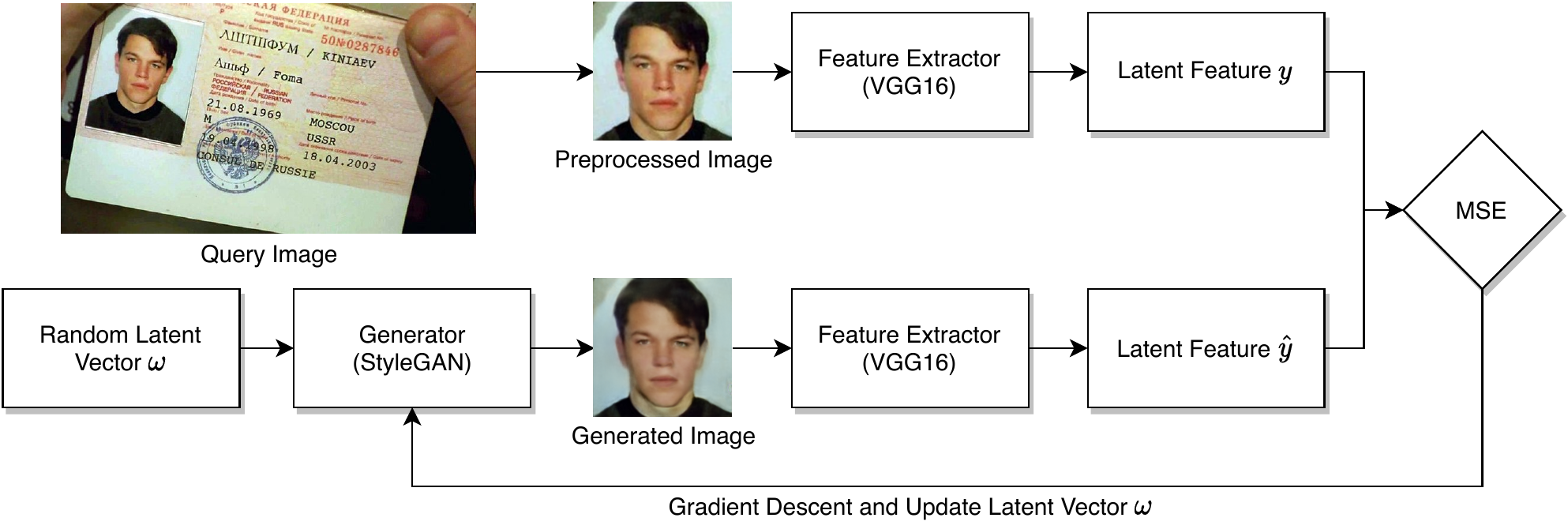}
\caption{
    Feature optimization pipeline to encode an arbitrary image.
    We first crop and align the query image and compare it with a random image generated by StyleGAN~\cite{karras2019style} using a loss function (e.g., mean squared error).
    Instead of comparing the two images in a pixel-wise fashion, we leverage a deep feature extractor (i.e., VGG16~\cite{simonyan2014very}) to obtain latent representations of the images, and then apply gradient descent to optimize the latent code $\mathbf{\hat{y}}$ of the random image until the synthesized image is close enough to the query one. 
}
\label{fig:feature_optimization_pipeline}
\end{figure*}

\section{Related Work}\label{sec:literature}

Our work bears some similarity to differentially private synthetic data generation~\cite{zhang2017privbayes,jordon2018pate,acs2018differentially,CensusDataCompetition}, perhaps most notably the use of generative networks for synthetic data~\cite{zhu2019learning,torkzadehmahani2019dp}.
However, the problem addressed in those works is generating synthetic data representing a \emph{set} of individuals, rather than the local differential privacy we achieve.  If applied directly to an image, such approaches would provide privacy for \emph{pixels}, not for persons - similar to the work of \cite{fan2018image,li2019anonymousnet} discussed in the introduction.
We also noted other transformation-based approaches~\cite{fan2019practical,ren2018learning,sun2018natural} that do not provide formal privacy guarantees.

Other work has shown that ML models can memorize (and subsequently leak) parts of their training data~\cite{song2017machine}. 
This can be exploited to expose private details about members of the training dataset~\cite{fredrikson2015model}. 
These attacks have spurred a push towards differentially private model training~\cite{abadi2016deep}, which uses techniques from the field of differential privacy to protect sensitive characteristics of training data.
This is a very different problem, our goal is to protect images that are \emph{not} contained in the training data.

There is also work targeted to defeating existing face recognition systems.
Much of the work in image privacy results in substantial distortion.  As with pixelization, these often produce images that are not visually pleasing.  Methods include
distorting images to make them unrecognizable~\cite{li2019anonymousnet,sun2018hybrid,wu2018privacy}, and
producing adversarial patches in the form of bright patterns printed on sweatshirts or signs, which prevent facial recognition algorithms from even registering their wearer as a person~\cite{thys2019fooling,wu2019making}.
Finally, given access to an image classification model, ``clean-label poison attacks'' can cause the model to misidentify a single, pre-selected image~\cite{shafahi2018poison,zhu2019transferable}.
However, these are targeted against facial recognition systems designed without regard to the privacy protection, and could be subject to targeted re-identification attacks such as \cite{gross2005integrating,gross2006model,gross2009face}.

\section{Differentially Private Imaging}\label{sec:method}
From the above, it should now seem obvious how we can get differential privacy:  Add noise to the latent vector $\z$ in a way that satisfies differential privacy.  This leaves three questions, addressed in this section.  The first is what mechanism do we use to add noise?  There are multiple mechanisms providing differential privacy; the right choice depends on how noise impacts the final result.  The second question is how much noise do we need to add?  This requires understanding the \emph{sensitivity} of the latent vector $\z$:  How much it can vary across different input images.
This will be covered in \cref{subsec:privacy_mechanism}.
The final issue is how to obtain the latent vector $\z$ in the first place?
Answering these questions requires a better understanding of the latent space.

\subsection{Latent Space and Image Encoding}\label{subsec:feature_optimization}
It has been widely observed that when linearly interpolating two latent codes $\z_1$ and $\z_2$, the appearance of the corresponding synthesis changes continuously~\cite{radford2015unsupervised,brock2018large,karras2019style}.
\cite{radford2015unsupervised} and \cite{bojanowski2017optimizing} identifies some vector arithmetic phenomenon in a GAN's latent space,
such as addition and subtraction invariance, implying the linearity property of latent spaces,
as well as monotonicity and Euclidean distance.
\cite{shen2019interpreting} provides a proof.
The linear interpolation between $\z_1$ and $\z_2$ forms a direction in latent space $\Z$, which further defines a hyperplane,
and the hyperplane splits a binary semantics.

\begin{figure*}[t]
    \begin{center}
    \begin{subfigure}{\sizeImage\textwidth}
        \centering
        \includegraphics[width=\textwidth]{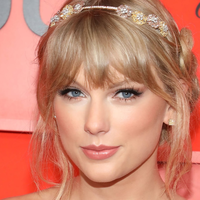}
    \end{subfigure}
    \Large{$\Rightarrow$}
    \begin{subfigure}{0.32\textwidth}
        \centering
        \includegraphics[width=\textwidth]{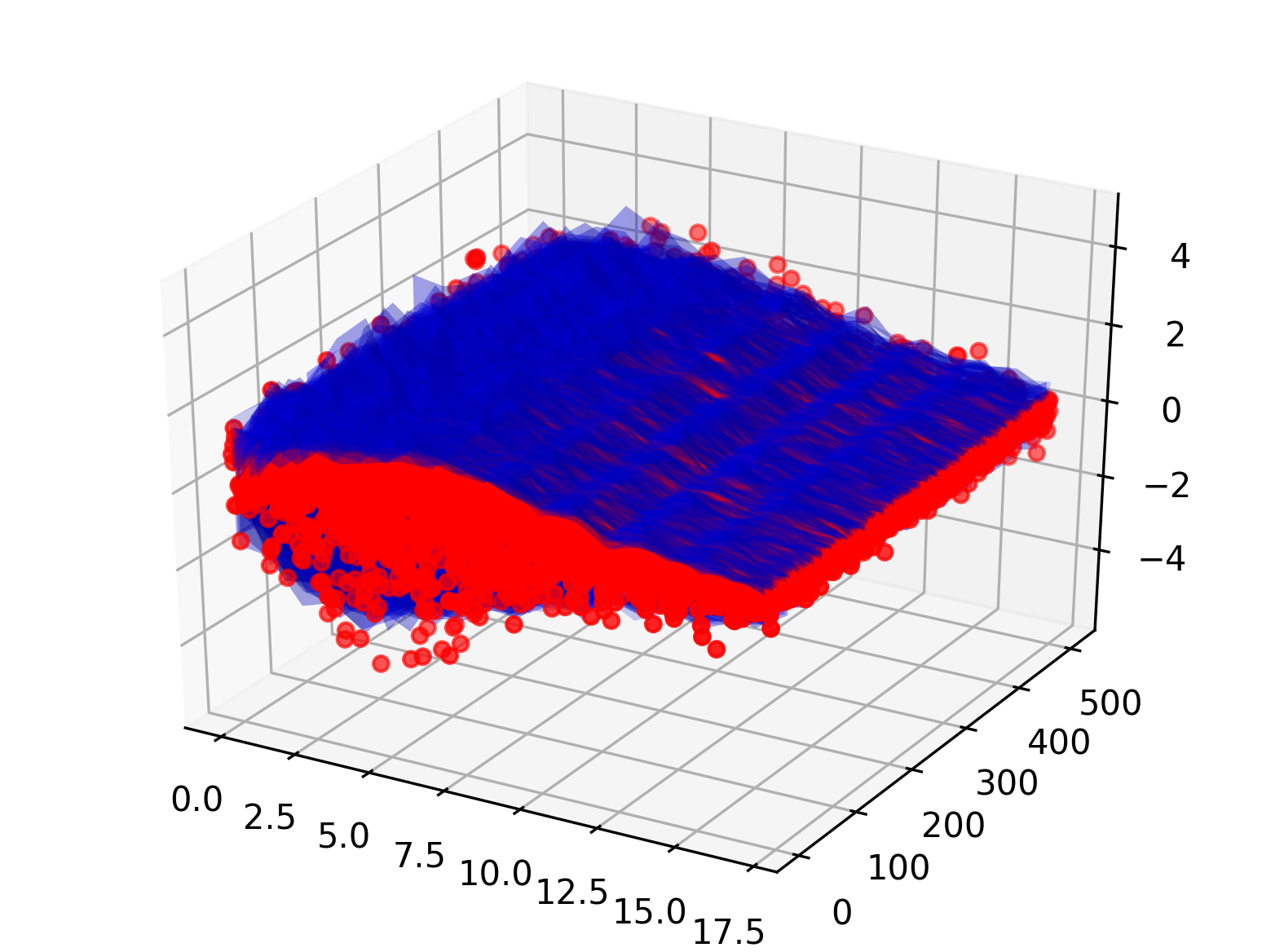}
    \end{subfigure}
    \Large{$\Rightarrow$}
    \begin{subfigure}{\sizeImage\textwidth}
        \centering
        \includegraphics[width=\textwidth]{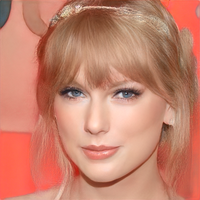}
    \end{subfigure}
    \begin{subfigure}{\sizeImage\textwidth}
        \centering
        \includegraphics[width=\textwidth]{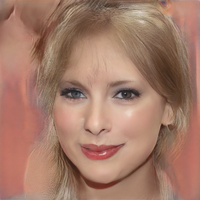}
    \end{subfigure}
    \begin{subfigure}{\sizeImage\textwidth}
        \centering
        \includegraphics[width=\textwidth]{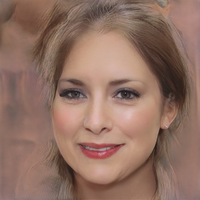}
    \end{subfigure}
    \end{center}
    \begin{center}
    \begin{subfigure}{\sizeImage\textwidth}
        \centering
        \includegraphics[width=\textwidth]{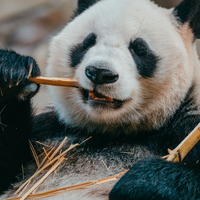}
    \end{subfigure}
    \Large{$\Rightarrow$}
    \begin{subfigure}{0.32\textwidth}
        \centering
        \includegraphics[width=\textwidth]{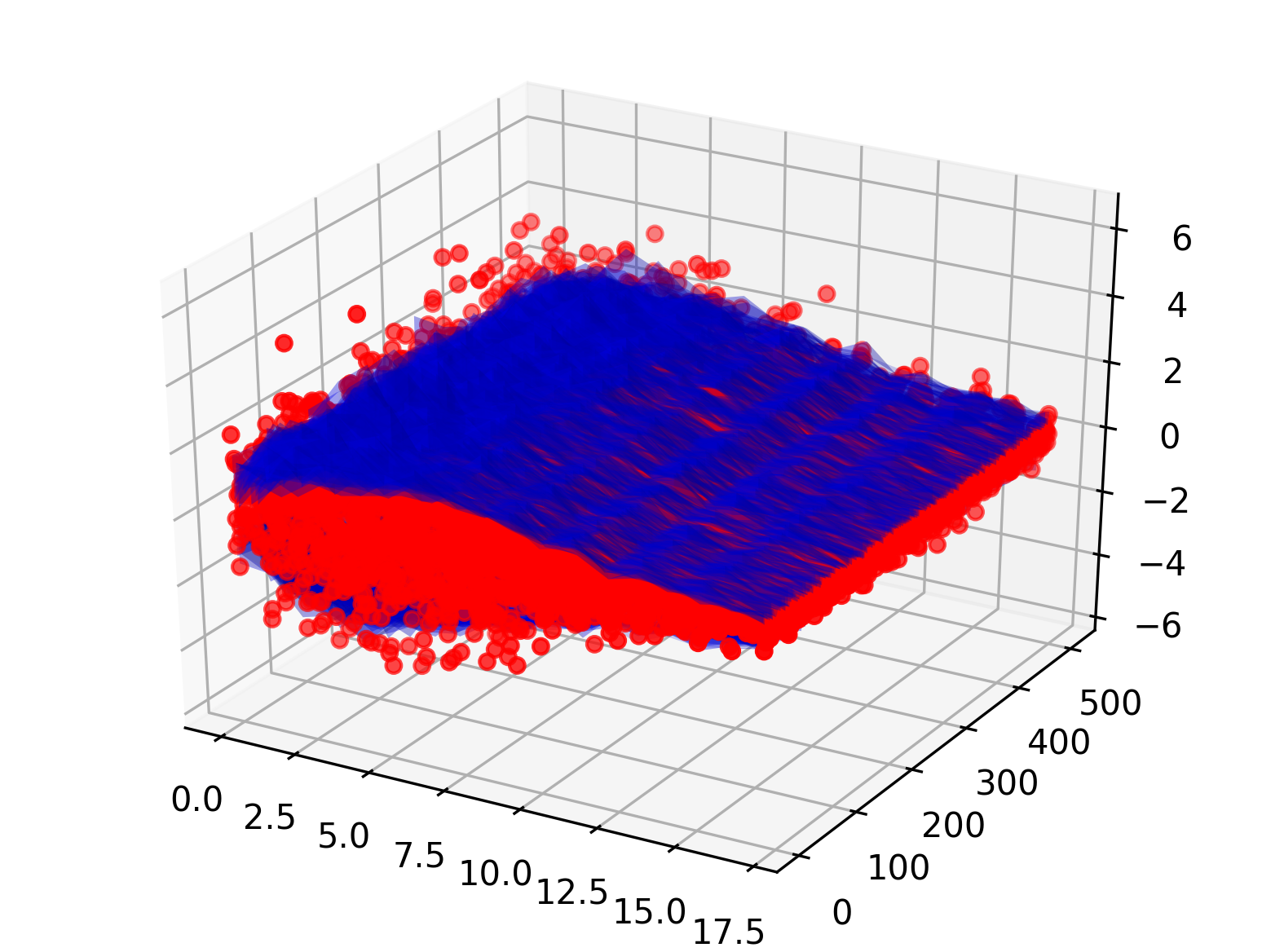}
    \end{subfigure}
    \Large{$\Rightarrow$}
    \begin{subfigure}{\sizeImage\textwidth}
        \centering
        \includegraphics[width=\textwidth]{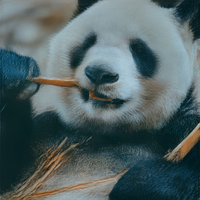}
    \end{subfigure}
    \begin{subfigure}{\sizeImage\textwidth}
        \centering
        \includegraphics[width=\textwidth]{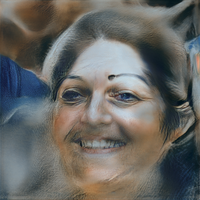}
    \end{subfigure}
    \begin{subfigure}{\sizeImage\textwidth}
        \centering
        \includegraphics[width=\textwidth]{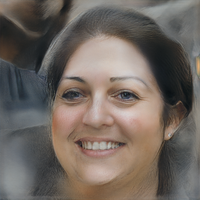}
    \end{subfigure}
    \end{center}
\caption{Clipping in the latent space. Given an input image, we first obtain its 
    latent code ({\color{red} red points}) using the feature optimization pipeline 
    in \cref{fig:feature_optimization_pipeline}. We then clip the latent code
    to make all the components fall within the lower and upper bounds ({\color{blue} blue surfaces})
    given by their distributions in the dataset.% (see \cref{fig:dlatents_heatmap}).
    The clipped latent code are fed into the GAN model to generate the outputs.
    Note that none of the examples are in the training dataset.
    The three outputs (from left to right) are clipping at $0\%/100\%$, $15\%/85\%$, and $30\%/70\%$, respectively.
    }
\label{fig:dlatents_clipping}
\end{figure*}

Given a real world facial image, the problem of finding its corresponding latent representation can be considered as an optimization problem where we search the latent space to find a latent vector, from which the reconstructed image is close enough (and hopefully identical) to the query image.
\Cref{fig:feature_optimization_pipeline} illustrates the optimization pipeline.
Being trained in a reasonably large and representative image dataset (e.g., FlickrFaces-HQ (FFHQ)  \cite{karras2019style}), a GAN model is presumed to memorize and represent the universe of facial images.
We first start with a random latent vector $\omega$ and place it in a generator (e.g., StyleGAN~\cite{karras2019style}) to obtain a synthesized image.
Instead of comparing the query image with its synthesized counterpart in a pixel-wise manner, we leverage a feature extractor (VGG16~\cite{simonyan2014very}) to obtain latent representations of each image and compare the loss function (i.e., MSE) in the feature space, as deep feature loss has been shown superior to pixel loss in practice~\cite{zhang2018unreasonable}.
Afterward, we use gradient descent to update the latent vector $\omega$ until the generated image converges to the query one.

\subsection{Privacy Mechanism}\label{subsec:privacy_mechanism}
A key issue in using the Laplace mechanism for $\varepsilon$-differential privacy is determining the sensitivity:  How much changing one individual %in a dataset
can impact the result.
Sensitivity is the maximum amount that the latent space could change by replacing one image to e privatized with \textbf{any} other image.

More formally, we want to determine the maximum the sensitive values could change if we replaced any possible input image with any other input image.  This would be the \emph{LDP sensitivity}, and adding noise commensurate with that difference would give us $\varepsilon$-differential privacy.

Unfortunately, the maximum possible difference in the latent space between any two input images is not only difficult to bound, but would result in untenable levels of noise.  Imagine, for example, a completely black and completely white image - vastly different, and not really interesting from a privacy point of view.  Furthermore, while those may be the greatest difference in the original space, we need to determine the greatest difference in the latent space, which is not directly related to pixel-level differences in the input.

\begin{algorithm}[t]
\caption{Sensitivity Calculation.}
\label{alg:sensitivity}
\begin{algorithmic}[1]
\REQUIRE{Dataset with identities $D = \{(X^{(i)}, id_i)\}_{i=1}^{n}$;}
\REQUIRE{Encoder $f: \X \rightarrow \Z$, where $\X$ represents the image space and $\Z \subseteq \R^m$ is the latent space with $m$ latent semantics;}
\FOR{each image $X^{(i)}$}
    \STATE{latent vector $Z^{(i)} \leftarrow f(X^{(i)})$;}
\ENDFOR
\FOR{each latent semantics $Z_j$}
    \STATE{local sensitivity $LS_j \leftarrow \underset{d(x, y) \le 1}{\max} || f(x) - f(y) ||_1$;}
\ENDFOR
\STATE{LDP sensitivity $S_L \leftarrow \max\limits_{x}~LS_{f}(x)$;}
\end{algorithmic}
\end{algorithm}

We use an idea from \emph{maximum observed sensitivity}\cite{MOS}.  They use the largest sensitivity observed across a dataset of significant size as a surrogate for the range of possible values.  In our case, the training data (for which we aren't concerned about privacy) is the large dataset; the maximum difference between any two images in the latent space could stand in for the range of possible values.  Unfortunately, this does not provide $\varepsilon$-differential privacy:  If we were given an unusual input image (say, someone standing on their head, or with particularly unusual features) it could fall outside these bounds, and result in a recognizable image.

Instead, we use the maximum observed sensitivity to \emph{clip} images in the latent space.
Any values that fall outside the observed bounds are clipped to the observed bounds, guaranteeing that the range of the input to the differential privacy mechanism is known, allowing us to determine sensitivity.
This allows us to fully satisfy $\varepsilon$-differential privacy.

Clipping the images does not come without cost.  While it ensures we satisfy differential privacy, an image that falls outside the ``normal'' values observed in the training data may be significantly distorted.
We show examples in \cref{fig:dlatents_clipping}.
The original image is on the left, followed by the 3D point cloud visualization of the latent code in the second column;
The third column are images clipped to the maximum and minimum values (i.e., $0\%$ and $100\%$) observed in a sample of 3500 of the training data images; 
the fourth and fifth columns are clipped at $15\%$ and $85\%$, and $30\%$ and $70\%$, respectively.

\begin{figure}[t]
\centering
\includegraphics[width=\sizeImage\textwidth]{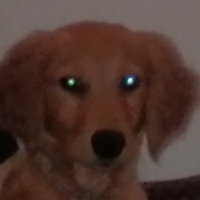}
~
\includegraphics[width=\sizeImage\textwidth]{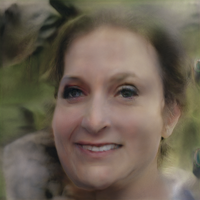}
\caption{
\textbf{Left} is the dog in \cref{fig:example},
and \textbf{right} is its differentially private face.
}
\label{fig:dpdog}
\end{figure}

Note that we do not claim the clipping itself provides differential privacy.
It enables us to bound the range of the input to the mechanism, so that the Laplace mechanism can be used to satisfy differential privacy.
In particular, it enables us to provide privacy for outliers.
Clipping forces outliers into the range of the training data, causing an image to be generated that resembles the training data when the actual input image is far from the training data.
An example is privatizing the dog from \cref{fig:example}, giving \cref{fig:dpdog}.
This results in very high distortion for outlying images, perhaps suggesting they are outliers, but still satisfying the formal definition of differential privacy.  It allows us to satisfy differential privacy with higher fidelity for images that bear closer resemblance to the training data.

\subsection{Algorithm}\label{subsec:algorithm}
We can now discuss how we provide $\varepsilon$-local differential privacy.  The idea is that the privacy budget $\varepsilon$ is divided among the various components in the latent space.  Each is used, along with the sensitivity derived from the clipping values for that component (based solely on the public training data), to determine a random draw of Laplace noise for that component,
which is again clipped (a postprocessing step).
This gives an $\varepsilon$-differentially private version of the image \emph{in the latent space}.  

\begin{algorithm}[t]
\caption{Differentially Private Imaging with Laplace Mechanism}
\label{alg:DP_imaging}
\begin{algorithmic}[1]
\REQUIRE{Input image  $X^{(i)}$;}
\REQUIRE{Encoder $f: \X \rightarrow \Z$;} %, where $\X$ represents the image space and $\Z \subseteq \R^m$ is the latent space with $m$ latent semantics;}
\REQUIRE{Generator $g:\Z \rightarrow \X$;}
\REQUIRE{Latent space sensitivities ${S_L}_j$;}
\REQUIRE{Latent space weights $w_j$ s.t. $\sum w_j = 1$;}
\REQUIRE{Privacy parameter $\varepsilon$;}
\REQUIRE{Laplace distribution $Lap(0, \lambda)$;} % centered at $0$ with scale parameter $\lambda$;}
\REQUIRE{Clipping function $f_c (i, j, \alpha)$;}
\STATE{latent vector $Z^{(i)} \leftarrow f(X^{(i)})$;}
\FOR{each latent semantics $Z^{(i)}_j$}
    \STATE{obtain a random $\delta$ from $Lap({S_L}_j \cdot w_j / \varepsilon)$;}
    \STATE{$Z'^{(i)}_j \leftarrow Z^{(i)}_j + \delta$;}
    \STATE{$Z''^{(i)}_j \leftarrow f_c ( Z'^{(i)}_j )$;}
\ENDFOR
\STATE{desired noisy image $X'^{(i)} \leftarrow g(Z''^{(i)})$;}
\end{algorithmic}
\end{algorithm}

We use this differentially private latent space version, with no reference to the original image or the latent space transformation of the original image, to generate an image using the previously described generative network.
The overall algorithm is given in \cref{alg:DP_imaging}.
Note that the feature optimization pipeline in \cref{fig:feature_optimization_pipeline} serves as the encoder $f$.

\subsection{Privacy Guarantee}\label{subsec:proof}
Remember that our goal is not to protect the individuals in the training data (these are assumed to be public, e.g., for the experiments in this paper the encoder and generator were trained using the FlickrFaces-HQ (FFHQ) dataset.)  The goal is to protect the individual in a \emph{new} image.  Therefore we assume that nothing in \cref{alg:DP_imaging} depends on the individual in the input image $X^{(i)}$ except what is explicitly shown in the algorithm.

\begin{theorem}\label{thm:proof}
    \Cref{alg:DP_imaging} provides $\varepsilon$-local differential privacy.
\end{theorem}
\begin{proof}
$\mathcal{M}$ is the randomized mechanism in \cref{alg:DP_imaging}.
Using the notations in \cite{erlingsson2014rappor} and above, we have
\begin{align*}
    \frac{\Pr[\mathcal{M}(v, f, \varepsilon) = s]}{\Pr[\mathcal{M}(v', f, \varepsilon) = s]}
    & = \frac{\Pr [Lap(S_L \cdot w_f / \varepsilon)] = s - f(v)}{\Pr [Lap(S_L \cdot w_f / \varepsilon)] = s - f(v')} \\
    & = \frac{\frac{S_L \cdot w_f}{\varepsilon} \cdot \exp({ - \frac{|s - f(v)| \varepsilon}{S_L \cdot w_f} })}{\frac{S_L}{\varepsilon} \cdot \exp({ -   \frac{|s - f(v')| \varepsilon}{S_L \cdot w_f} })} \\
    & = \exp(\frac{\varepsilon | f(v') - f(v) |}{S_L \cdot w_f}) \le \exp(\varepsilon) \\
    & = \exp(\frac{\varepsilon | f(v') - f(v) |}{S_L}) \le \exp(\varepsilon \cdot w_f)
\end{align*}
\end{proof}

Each component in the semantic space transformation of the image has noise added from a Laplace distribution.
From \cite{dwork2006calibrating}, we have that each component $Z'^{(i)}_j$ is $(\varepsilon \cdot w_f)$-differentially private.
Sequential composition gives $Z'^{(i)}$ is $\sum \varepsilon \cdot w_j = \varepsilon \sum w_j$ differentially private.  Since $\sum w_j = 1$, this shows that $Z'^{(i)}$ is $\varepsilon$-differentially private.  The remaining image generation step uses only the \emph{noise-added} version of the image in the semantic space.  Since no information from the individual in question is used in this or the generator $g$, the postprocessing theorem of differential privacy %\cref{thm:postprocessing}
tells us that the output image is $\varepsilon$-differentially private.

Some may ask why we do not use a parallel composition argument, since the noise is added independently to each component.  The problem is that parallel composition requires that the components be from disjoint individuals; this would be like saying we want to avoid identifying an individual's hairstyle and smile, rather than protecting against identifying the individual.

Note that this makes the assumption that not only is the image to be protected not in the training data, but that the \emph{individual pictured} is not in the training data (or more specifically, not in the data used to train the image generator $g$.)

\section{Empirical Results}\label{sec:experiment}
The previous section shows how we can achieve a differentially private image, and why we might expect it to produce reasonable images..
We evaluate the proposed method and its results, both qualitatively and quantitatively, using real world images.

\subsection{Dataset}
For experiments, we adopt the FlickrFaces-HQ (FFHQ) dataset~\cite{karras2019style} collected by NVIDIA, consisting of 70,000 high-resolution ($1024 \times 1024$) human facial images.
This dataset covers a wide spectrum of human faces, including variations in age, ethnicity, and image backgrounds; crawled from Flickr.
To resolve computational issues, we use a
randomly-selected
subset 
of 3500 images for these experiments.
We align and crop the images using Dlib\footnote{\url{http://dlib.net/}}.
All results reported in this paper %, including numbers and figures,
are based on the aligned and cropped dataset.
All the images shown in the paper are publicly available or taken by the authors, and none of them are in the training dataset.

\subsection{Evaluation}

We first show what happens with small values of $\varepsilon$. \Cref{fig:small_epsilon} demonstrates $\varepsilon=1, \dots, 5000$; we can see that the images are not very useful.  (Although with sufficient clipping, they are recognizably human.)
Note that $\varepsilon=1$ is very strict; a way to think of this is that if the adversary knows the image is of either Barack Obama or Hillary Clinton, we are adding enough noise that the adversary's best guess would be right 75\% of the time (as opposed to 50\% without seeing the image.)  This does hold for these images, even knowing that \cref{fig:small_epsilon} is either Obama or Clinton, anyone (correctly) guessing that it was generated from an image of President Obama would have little confidence in that guess.

Our use case is much different; the image may be known to be part of a large community (e.g., it is taken on a college campus, so likely belongs to someone on that campus), but could belong to any of thousands of people in that group.  This enables a much larger $\varepsilon$ without a significant risk of re-identification \emph{in the absence of other information providing a strong expectation on who the image belongs to.}  For more discussion of setting $\varepsilon$, see \cite{DiffIdKDD12,LiMembership}.

The remainder of the images we show are with much larger $\varepsilon$.
For example, in \cref{fig:teaser}, if you knew the names of the authors, you would have a good shot at guessing which picture went with which author.  But in a double-blind review process, where the authors could be any of the thousands of people who might submit to Privacy Enhancing Technologies, identifying who the authors are is infeasible even at $\varepsilon=9216$.

\begin{figure}[t]
\begin{center}
\begin{minipage}[b]{\sizeImage\textwidth}\centering
    $\bm{\varepsilon = 1}$ \\
    \vspace{1mm}
    \includegraphics[width=\textwidth]{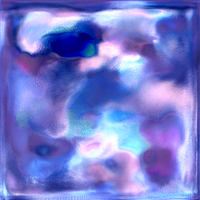}
\end{minipage}
\hspace{\sizeHorizontalSpace}
\begin{minipage}[b]{\sizeImage\textwidth}\centering
    $\bm{\varepsilon = 1000}$ \\
    \vspace{1mm}
    \includegraphics[width=\textwidth]{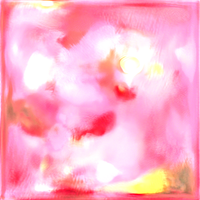}
\end{minipage}
\hspace{\sizeHorizontalSpace}
\begin{minipage}[b]{\sizeImage\textwidth}\centering
    $\bm{\varepsilon = 5000}$ \\
    \vspace{1mm}
    \includegraphics[width=\textwidth]{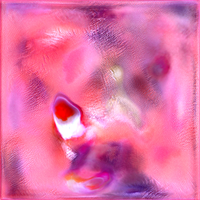}
\end{minipage}
\\
\begin{minipage}[b]{\sizeImage\textwidth}\centering
    \includegraphics[width=\textwidth]{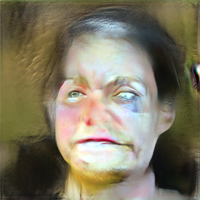}
\end{minipage}
\hspace{\sizeHorizontalSpace}
\begin{minipage}[b]{\sizeImage\textwidth}\centering
    \includegraphics[width=\textwidth]{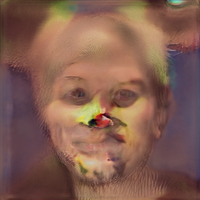}
\end{minipage}
\hspace{\sizeHorizontalSpace}
\begin{minipage}[b]{\sizeImage\textwidth}\centering
    \includegraphics[width=\textwidth]{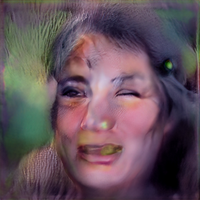}
\end{minipage}
\end{center}
\caption{Examples of
image privatized using small values of $\varepsilon$, providing privacy against an adversary with strong expectations of who the image might be.
The first row has no clipping; the second row is clipped at $10\%/90\%$.
In comparison, \cref{fig:privacy_budget} uses larger values of $\varepsilon$,
which leads to better perceptual quality.
}
\label{fig:small_epsilon}
\end{figure}

\begin{figure*}[hbtp]
\begin{center}
\setlength{\tabcolsep}{.15em}
\begin{tabular}{ccccccc}
    $Original$ & $\varepsilon_{ij} = \infty$ & $\varepsilon_{ij} = 64$ & $\varepsilon_{ij} = 32$ & $\varepsilon_{ij} = 16$ & $\varepsilon_{ij} = 8$ & $\varepsilon_{ij} = 4$ \\
    \includegraphics[width=\sizeImage\textwidth]{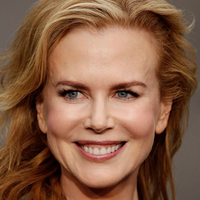}
    & \hspace{\sizeHorizontalSpace}
    \includegraphics[width=\sizeImage\textwidth]{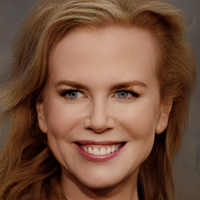}
    & \hspace{\sizeHorizontalSpace}
    \includegraphics[width=\sizeImage\textwidth]{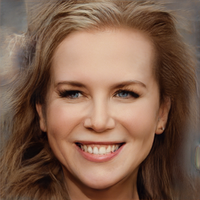}
    & \hspace{\sizeHorizontalSpace}
    \includegraphics[width=\sizeImage\textwidth]{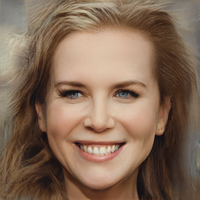}
    & \hspace{\sizeHorizontalSpace}
    \includegraphics[width=\sizeImage\textwidth]{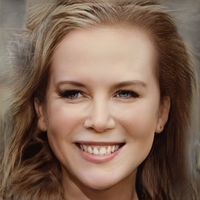}
    & \hspace{\sizeHorizontalSpace}
    \includegraphics[width=\sizeImage\textwidth]{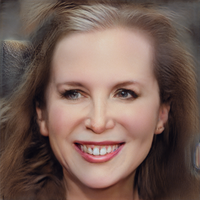}
    & \hspace{\sizeHorizontalSpace}
    \includegraphics[width=\sizeImage\textwidth]{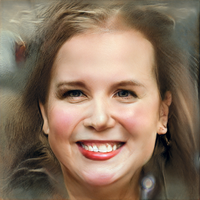}
    \\[-3.5pt]
    \includegraphics[width=\sizeImage\textwidth]{figures_teaser_kidman_cropped.png}
    & \hspace{\sizeHorizontalSpace}
    \includegraphics[width=\sizeImage\textwidth]{figures_teaser_kidman_optimized.png}
    & \hspace{\sizeHorizontalSpace}
    \includegraphics[width=\sizeImage\textwidth]{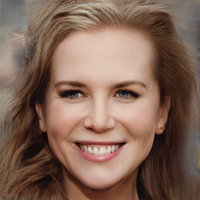}
    & \hspace{\sizeHorizontalSpace}
    \includegraphics[width=\sizeImage\textwidth]{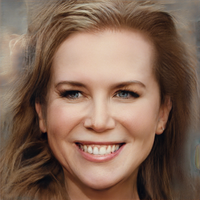}
    & \hspace{\sizeHorizontalSpace}
    \includegraphics[width=\sizeImage\textwidth]{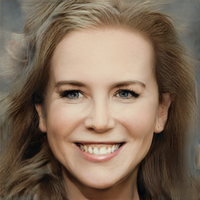}
    & \hspace{\sizeHorizontalSpace}
    \includegraphics[width=\sizeImage\textwidth]{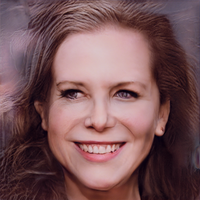}
    & \hspace{\sizeHorizontalSpace}
    \includegraphics[width=\sizeImage\textwidth]{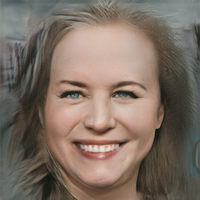}
    \\[-1pt]
    \includegraphics[width=\sizeImage\textwidth]{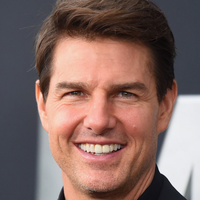}
    & \hspace{\sizeHorizontalSpace}
    \includegraphics[width=\sizeImage\textwidth]{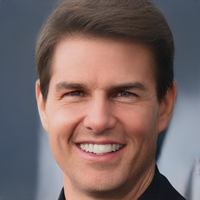}
    & \hspace{\sizeHorizontalSpace}
    \includegraphics[width=\sizeImage\textwidth]{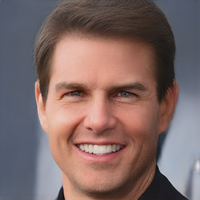}
    & \hspace{\sizeHorizontalSpace}
    \includegraphics[width=\sizeImage\textwidth]{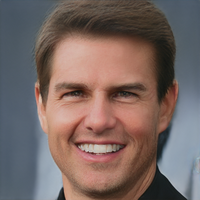}
    & \hspace{\sizeHorizontalSpace}
    \includegraphics[width=\sizeImage\textwidth]{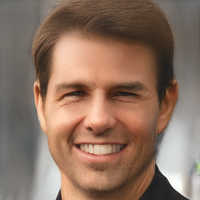}
    & \hspace{\sizeHorizontalSpace}
    \includegraphics[width=\sizeImage\textwidth]{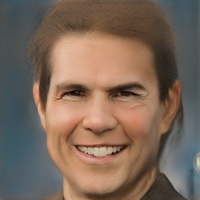}
    & \hspace{\sizeHorizontalSpace}
    \includegraphics[width=\sizeImage\textwidth]{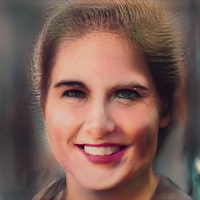}
    \\[-3.5pt]
    \includegraphics[width=\sizeImage\textwidth]{figures_teaser_cruise_cropped.png}
    & \hspace{\sizeHorizontalSpace}
    \includegraphics[width=\sizeImage\textwidth]{figures_teaser_cruise_optimized.png}
    & \hspace{\sizeHorizontalSpace}
    \includegraphics[width=\sizeImage\textwidth]{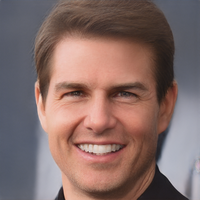}
    & \hspace{\sizeHorizontalSpace}
    \includegraphics[width=\sizeImage\textwidth]{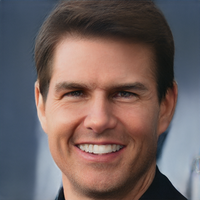}
    & \hspace{\sizeHorizontalSpace}
    \includegraphics[width=\sizeImage\textwidth]{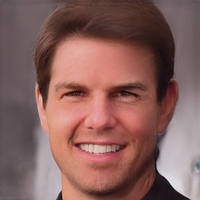}
    & \hspace{\sizeHorizontalSpace}
    \includegraphics[width=\sizeImage\textwidth]{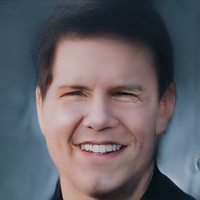}
    & \hspace{\sizeHorizontalSpace}
    \includegraphics[width=\sizeImage\textwidth]{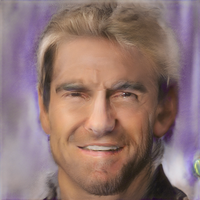}
    \\
    \multicolumn{7}{c}{Clipping at $12.5\%$ and $87.5\%$.}\\
\end{tabular}
\end{center}
\begin{center}
\setlength{\tabcolsep}{.15em}
\begin{tabular}{ccccccc}
    $Original$ & $\varepsilon_{ij} = \infty$ & $\varepsilon_{ij} = 16$ & $\varepsilon_{ij} = 8$ & $\varepsilon_{ij} = 4$ & $\varepsilon_{ij} = 2$ & $\varepsilon_{ij} = 1$ \\
    \includegraphics[width=\sizeImage\textwidth]{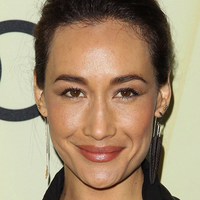}
    & \hspace{\sizeHorizontalSpace}
    \includegraphics[width=\sizeImage\textwidth]{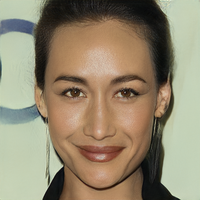}
    & \hspace{\sizeHorizontalSpace}
    \includegraphics[width=\sizeImage\textwidth]{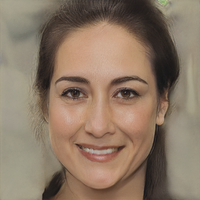}
    & \hspace{\sizeHorizontalSpace}
    \includegraphics[width=\sizeImage\textwidth]{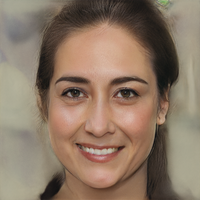}
    & \hspace{\sizeHorizontalSpace}
    \includegraphics[width=\sizeImage\textwidth]{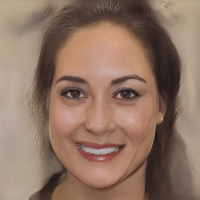}
    & \hspace{\sizeHorizontalSpace}
    \includegraphics[width=\sizeImage\textwidth]{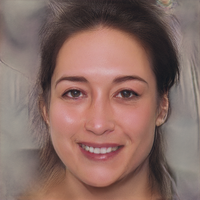}
    & \hspace{\sizeHorizontalSpace}
    \includegraphics[width=\sizeImage\textwidth]{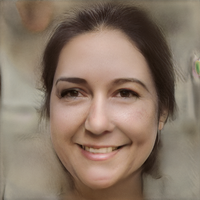}
    \\[-3.5pt]
    \includegraphics[width=\sizeImage\textwidth]{figures_teaser_nikita2_cropped.png}
    & \hspace{\sizeHorizontalSpace}
    \includegraphics[width=\sizeImage\textwidth]{figures_teaser_nikita2_optimized.png}
    & \hspace{\sizeHorizontalSpace}
    \includegraphics[width=\sizeImage\textwidth]{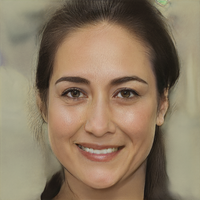}
    & \hspace{\sizeHorizontalSpace}
    \includegraphics[width=\sizeImage\textwidth]{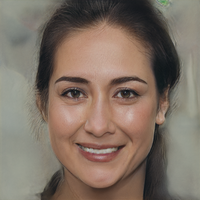}
    & \hspace{\sizeHorizontalSpace}
    \includegraphics[width=\sizeImage\textwidth]{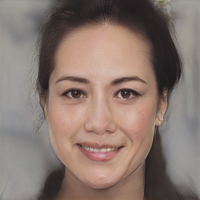}
    & \hspace{\sizeHorizontalSpace}
    \includegraphics[width=\sizeImage\textwidth]{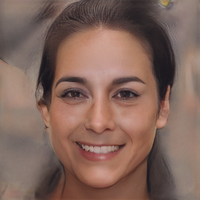}
    & \hspace{\sizeHorizontalSpace}
    \includegraphics[width=\sizeImage\textwidth]{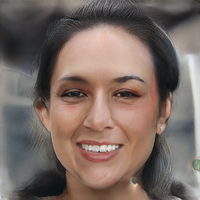}
    \\[-1pt]
    \includegraphics[width=\sizeImage\textwidth]{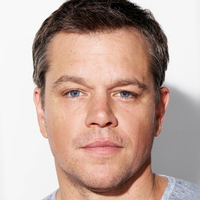}
    & \hspace{\sizeHorizontalSpace}
    \includegraphics[width=\sizeImage\textwidth]{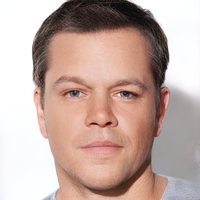}
    & \hspace{\sizeHorizontalSpace}
    \includegraphics[width=\sizeImage\textwidth]{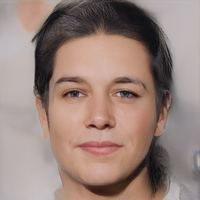}
    & \hspace{\sizeHorizontalSpace}
    \includegraphics[width=\sizeImage\textwidth]{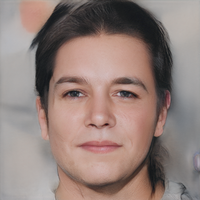}
    & \hspace{\sizeHorizontalSpace}
    \includegraphics[width=\sizeImage\textwidth]{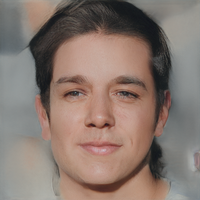}
    & \hspace{\sizeHorizontalSpace}
    \includegraphics[width=\sizeImage\textwidth]{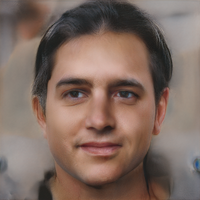}
    & \hspace{\sizeHorizontalSpace}
    \includegraphics[width=\sizeImage\textwidth]{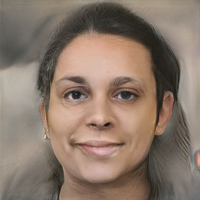}
    \\[-3.5pt]
    \includegraphics[width=\sizeImage\textwidth]{figures_teaser_bourne_cropped.png}
    & \hspace{\sizeHorizontalSpace}
    \includegraphics[width=\sizeImage\textwidth]{figures_teaser_bourne_optimized.png}
    & \hspace{\sizeHorizontalSpace}
    \includegraphics[width=\sizeImage\textwidth]{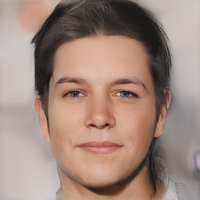}
    & \hspace{\sizeHorizontalSpace}
    \includegraphics[width=\sizeImage\textwidth]{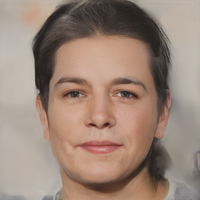}
    & \hspace{\sizeHorizontalSpace}
    \includegraphics[width=\sizeImage\textwidth]{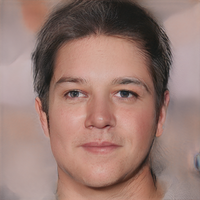}
    & \hspace{\sizeHorizontalSpace}
    \includegraphics[width=\sizeImage\textwidth]{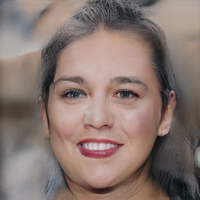}
    & \hspace{\sizeHorizontalSpace}
    \includegraphics[width=\sizeImage\textwidth]{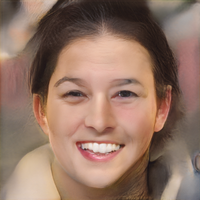}
    \\
    \multicolumn{7}{c}{Clipping at $25\%$ and $75\%$.}\\
\end{tabular}
\end{center}
\caption{Experimental results with different privacy and clipping settings. 
For each identity, two groups of experimental results under the same settings are given.
They produce different outputs because of the randomness of the mechanism.
Note that the number of latent components is $18 \times 512 = 9216$ for our experiments
and privacy loss $\varepsilon = \sum \varepsilon_{ij} = 9216 \cdot \varepsilon_{ij}$.
}
\label{fig:privacy_budget}
\end{figure*}

\begin{figure*}[t]
\centering
\includegraphics[width=0.32\textwidth]{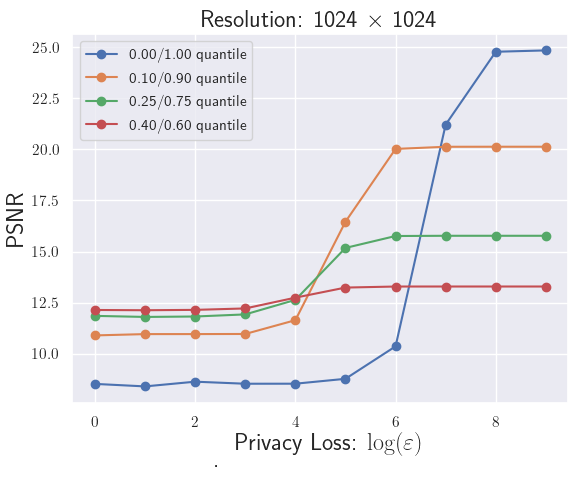}
\includegraphics[width=0.32\textwidth]{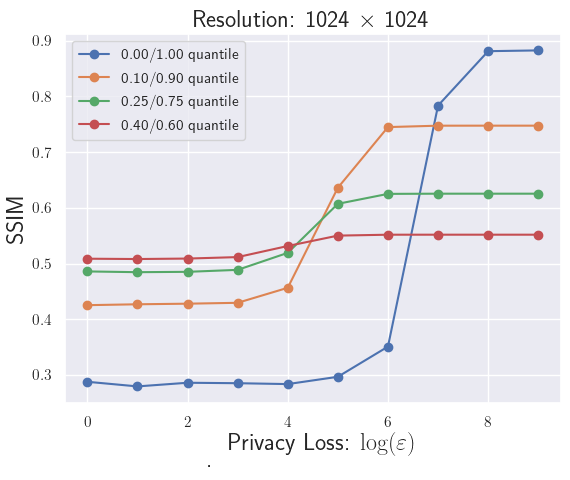}
\includegraphics[width=0.32\textwidth]{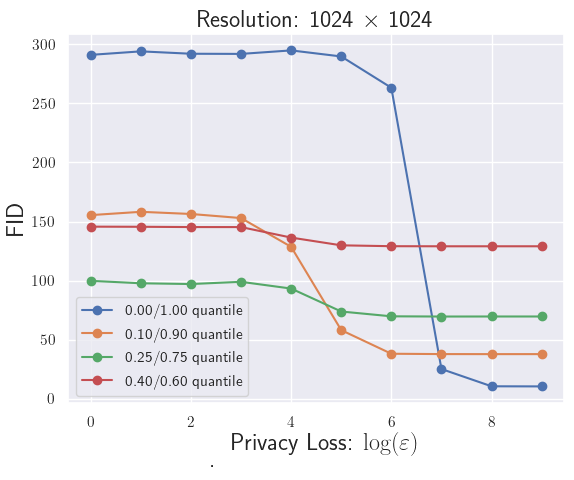}
\caption{
    Trade-offs between privacy and utility.
By varying privacy parameter $\varepsilon$ in the latent space, the pixels vary accordingly.
We show SSIM, PSNR, and FID with respect to privacy loss,
and the larger the privacy loss is, the closer the image to its original.
Above figures demonstrate that both pixel-wise distortion and perceptual distance become smaller as the privacy loss $\varepsilon$ increases,
indicating that the less noise added, the closer the generated image to its original identity, and vice versa.
This aligns with our intuition, as presumably each latent value controls a group of pixels.
}
\label{fig:metrics_merged}
\end{figure*}

With a larger of $\varepsilon$ and proper clipping, \cref{fig:privacy_budget} 
has much more visually pleasing results.
It shows results under various settings.
For each identity, we show two groups of experimental results under the same setting.
They have different outputs because of the randomness of the mechanism.
The first column is the inputs and the second column is the generated results from the image encoding pipeline,
where the synthesized images are optimized to be as close to the original as possible (i.e., $\varepsilon = \infty$).
The remaining columns showcase outputs under different noise levels.
With a large noise (i.e., a smaller $\varepsilon$), the output image is less similar to the original (i.e., more private)
while still maintaining some fidelity (e.g., it is still a human face sharing some features with the original).

Clipping also plays an important role in this process to  maintain the perceptual quality of the image.
Under the same noise level, a heavily clipped output has better visual quality than one without clipping, although it loses more specifics of the original.
Notice that in \cref{fig:dlatents_clipping}, an input image of an animal after clipping results in a human face.
Even a white noise input, with substantial clipping, appears to show a human face (since this is what the training data consists of.)
This basically shows what surfaces in the latent space look like.

\Cref{fig:metrics_merged} quantitatively evaluate the outputs from the proposed method
with  different privacy and clipping settings.
PSNR and SSIM measure the level of distortion at the pixel level;
while FID captures the differences at the semantic level.
Four clipping settings are tested, $0\%/100\%$, $10\%/90\%$, $25\%/75\%$, and $40\%/60\%$,
each of which corresponds to a line in the figures.
The trends are clear that clipping makes the outputs more robust to noise;
and those without clipping would have greater distortion as noise increase.
These results align with \cref{fig:privacy_budget} as well as our intuition.

\section{Conclusion}\label{sec:conclusion}
In this work, we provide the first meaningful formal definition of $\varepsilon$-differential privacy for images
by leveraging the latent space of images and Laplace mechanism.
A practical framework is presented to tackle real world images.
Experimental results show that the proposed mechanism is able to preserve privacy in accordance with privacy budget $\varepsilon$
while maintain high perceptual quality for sufficiently large values of $\varepsilon$.

For a practical example of such a mechanism, assume a differentially private high-$\varepsilon$ image is posted on a social media site.  A face recognition system to automatically tag people in the image may be able to correctly tag the poster, and friends of the poster -- subjects that the poster would probably not choose to anonymize anyway.  But even with high $\varepsilon$, attempts to identify others in the image who are anonymized, while significantly better than a random guess, would still have extremely low confidence.

There is still considerable room for improvement.  We have split the privacy budget evenly between components; varying this split may result in significantly better quality.  Varying privacy budget between semantic components could be used to adjust what is preserved (e.g., preserving pose or emotional state at the expense of lower fidelity to age or gender.)  Methods based on the exponential mechanism of differential privacy rather than noise addition are likely to provide more realistic images, but with less relationship to the original.  The key is that all of these build on the same basic concept:  Defining modifications in the latent space such that privacy is provided for the \emph{person}.

{\small
\bibliographystyle{ieee_fullname}
\bibliography{db.bib}
}

\end{document}